\documentclass[runningheads]{llncs}
\usepackage{graphicx}
\usepackage{comment}
\usepackage{amsmath,amssymb} 
\usepackage{color}

\begin{document}
\pagestyle{headings}
\mainmatter
\def\ECCVSubNumber{2086}  

\title{Semantic Change Pattern Analysis} 


\titlerunning{Semantic Change Pattern Analysis}
%
\author{Wensheng Cheng\inst{1} \and
Yan Zhang\inst{1} \and Xu Lei\inst{1} \and Wen Yang\inst{1} \and
Guisong Xia\inst{2}}
\authorrunning{W. Cheng et al.}
%
\institute{Electronic Information School, Wuhan University, Wuhan 430072, China \and
School of Computer Science, Wuhan University, Wuhan 430072, China
\\\email{\{cwsinwhu,zhangyanchn,leixuchn,yangwen,guisong.xia\}@whu.edu.cn}}
\maketitle

\begin{abstract}
Change detection is an important problem in vision field, especially for aerial images. However, most works focus on traditional change detection, i.e., where changes happen, without considering the change type information, i.e., what changes happen. Although a few works have tried to apply semantic information to traditional change detection, they either only give the label of emerging objects without taking the change type into consideration, or set some kinds of change subjectively without specifying semantic information. To make use of semantic information and analyze change types comprehensively, we propose a new task called semantic change pattern analysis for aerial images. Given a pair of co-registered aerial images, the task requires a result including both where and what changes happen. We then describe the metric adopted for the task, which is clean and interpretable. We further provide the first well-annotated aerial image dataset for this task. Extensive baseline experiments are conducted as reference for following works. The aim of this work is to explore high-level information based on change detection and facilitate the development of this field with the publicly available dataset.  

\end{abstract}
\section{Introduction}
Change detection is a long-standing research problem, with various applications in general vision field \cite{goyette2012changedetection,bilodeau2013change,wang2014cdnet,feng2015fine,park2019robust} and aerial image scope \cite{gressin2013semantic,taneja2013city,daudt2019multitask,revaud2019did}. As a popular application domain, change detection based on aerial images could be used to analyze the changes of the land surface during a period. It becomes the basis of automatic aerial image analysis. The general definition for change detection is to identify the areas where changes happen by jointly analyzing two registered images \cite{bruzzone2012novel}. The output is a binary image, in which ones denote pixels where changes happen and zeros represent pixels that remain unchanged. Hence it can be regarded as a binary dense classification problem. Fig. \ref{fig:topic}(a) and Fig. \ref{fig:topic}(b) are a pair of co-registered aerial images. Fig. \ref{fig:topic}(c) is the result of traditional binary change detection task for the pair of images.

\begin{figure}[!t]
	\begin{minipage}[b]{.49\linewidth}
		\centering
		\centerline{\includegraphics[width=\columnwidth]{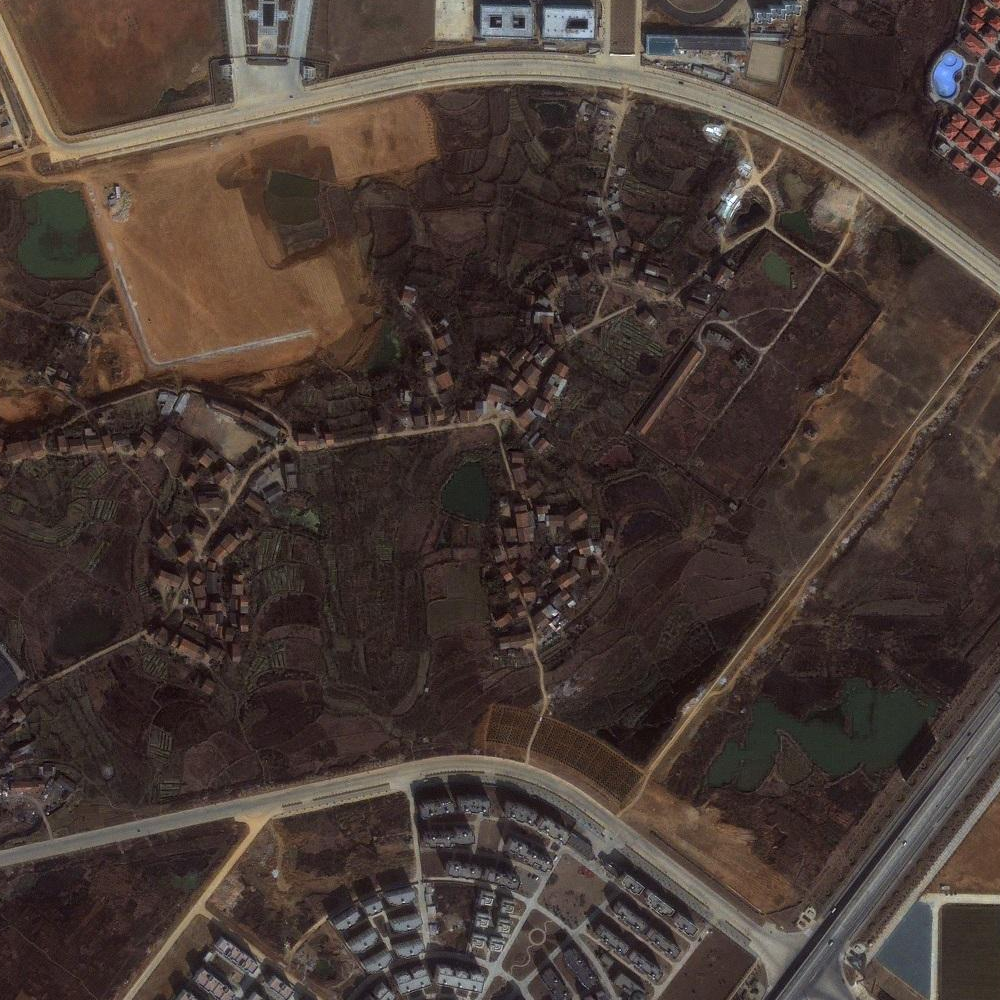}}
		\centerline{(a) Source image}
	\end{minipage}
	\hfill
	\begin{minipage}[b]{.49\linewidth}
		\centering
		\centerline{\includegraphics[width=\columnwidth]{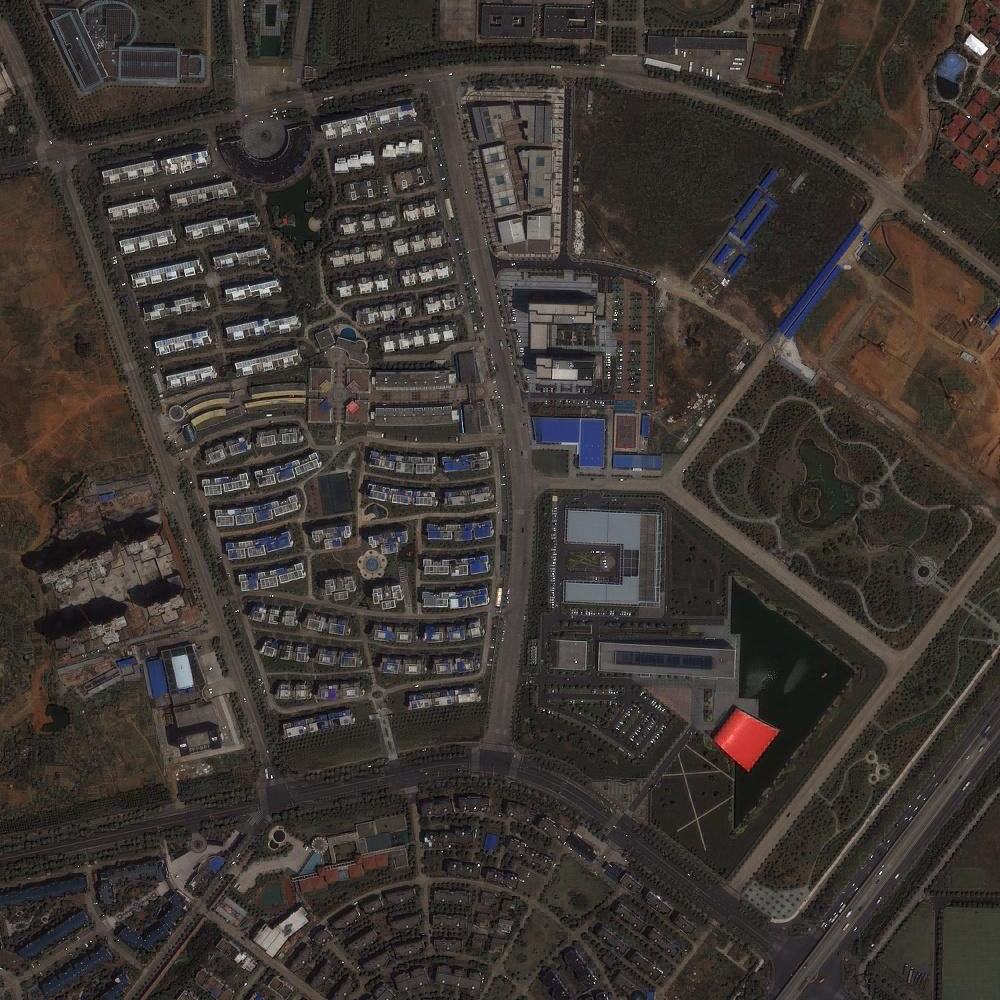}}
		\centerline{(b) Destination image}
	\end{minipage}
	\begin{minipage}[b]{.49\linewidth}
		\centering
		\centerline{\includegraphics[width=\columnwidth]{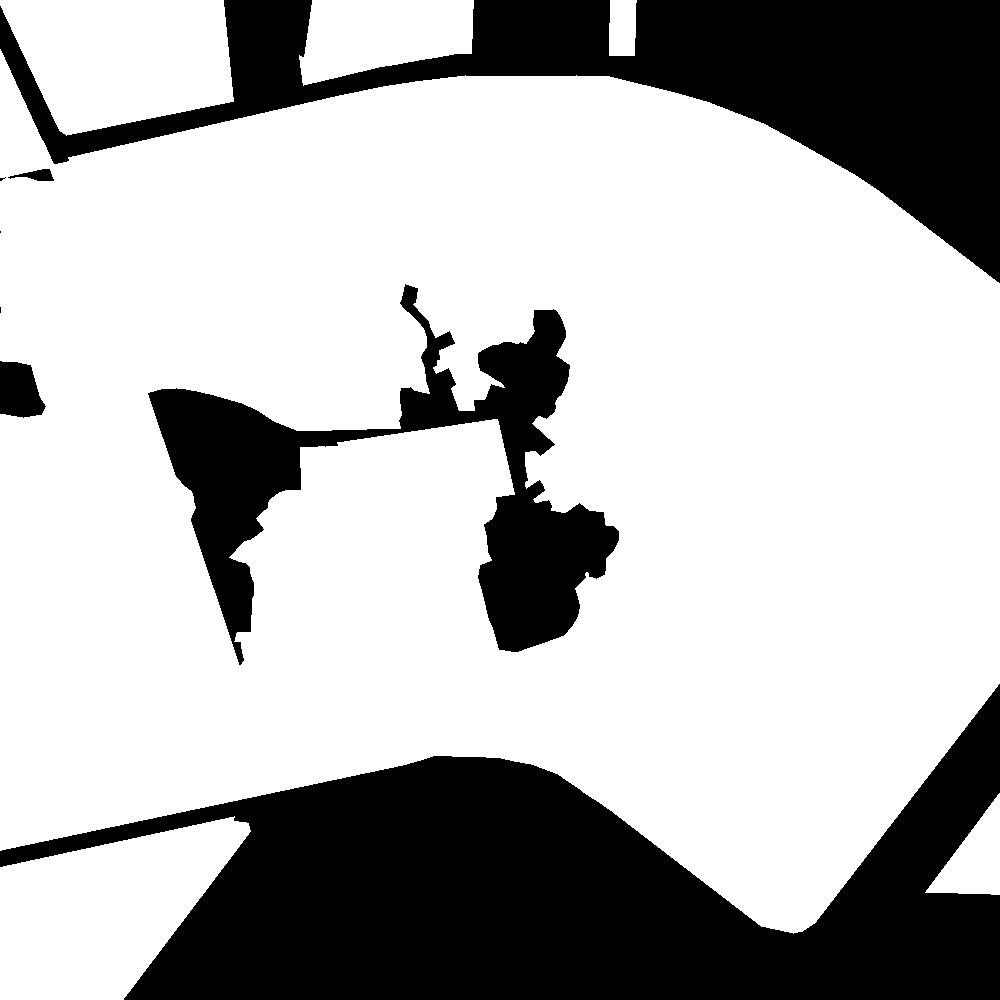}}
		\centerline{(c) Traditional change detection result}
	\end{minipage}
	\hfill
	\begin{minipage}[b]{.49\linewidth}
		\centering
		\centerline{\includegraphics[width=\columnwidth]{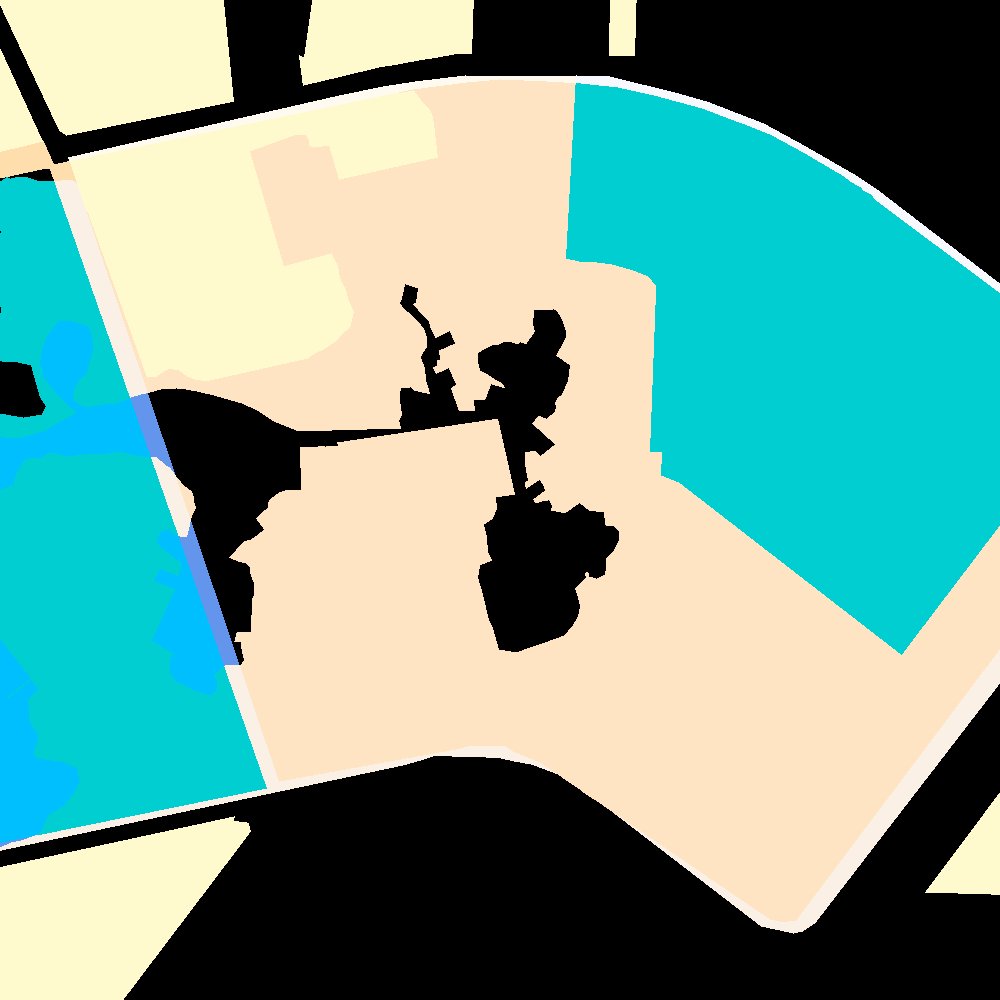}}
		\centerline{(d) Semantic change pattern analysis result}
	\end{minipage}
	\begin{minipage}[b]{\linewidth}
		\centering
		\centerline{\includegraphics[width=.99\columnwidth]{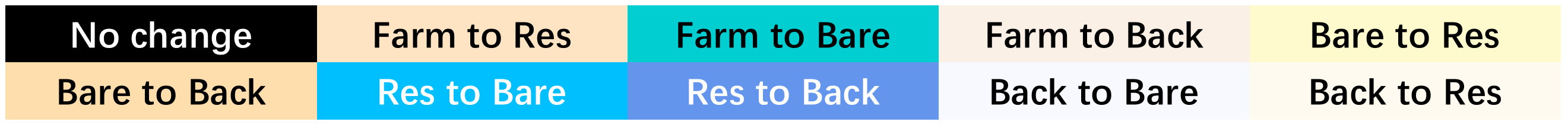}}
	\end{minipage}
	\caption{(a) is the source image before change. (b) is the registered destination image after change. (c) is the traditional change detetion task result. Black pixels are areas without change, and white pixels are areas where changes happen. (d) is proposed semantic change pattern analysis task result. Pixels in distinct colors represent various change types. In the color bar, Farm: farmland, Res: residential area, Bare: bare land, Back: background.}
	\label{fig:topic}
\end{figure}

However, there exist some drawbacks for the task. To begin with, traditional change detection methods only give the information about whether a region has changed, without the information about what kind of change it is. As a result, it can not fulfill the requirement for understanding the change type, which limits its application to a large extend. Besides, the definition of \emph{change} is really casual and vague. For instance, the appearance of grassland would change from summer to winter. Under this situation, whether we should label these areas as changed area is not clear at all. This would influence the following processing severely.

To deal with problems above, an efficient solution is to introduce semantic information. To our best knowledge, \cite{kataoka2016semantic} first proposed the concept named \emph{semantic change detection} for street images. It consists of two steps. The first step is the same with traditional change detection, which outputs a binary image showing changed areas. The second step is to simply label the newly added objects. Therefore, it can not decide the type of change, due to the lack of semantic information in the image before change. Following this definition, \cite{varghese2018changenet} proposed a new deep neural network model which performs well on public street view datasets. \cite{daudt2019multitask} tried to deal with semantic segmentation and change detection problems simultaneously through multitask learning. Although this work involves three kinds of change, city expansion, soil change, and water change, it does not include the information about what category the area changed from and which category the area changed to.

In our work, we propose a new task, semantic change pattern analysis(SCPA), to utilize semantic information and analyze change types comprehensively, especially for aerial images. Specifically, for a pair of co-registered images, the output of the task would be a pixel-level multi-class classification result. Let's refer to the image before change as \emph{source image}, and the image after change as \emph{destination image}. Each pixel of the output image would be assigned a class label, which denotes the change type for the corresponding pair of pixels. The \emph{change type} here is jointly defined by the pixel's semantic label in the source image and destination image, i.e., the pixel changes from \emph{what} to \emph{what}. For instance, label `1' denotes that the pixel has changed from bare land in source image to building in destination image. Note that if the source image and destination image are switched, i.e., the pixel changed from building to bare land, the change type is supposed to be different and another label other than `1' should be assigned. Fig. \ref{fig:topic}(d) shows the result of semantic change pattern analysis task.

To evaluate the performance for SCPA task, we need a simple and interpretable metric. Considering the fact that SCPA could be regarded as a multi-class pixel-level classification problem, where each class is a type of change, we adopt the mean Intersection over Union(IoU) over all classes as the main metric. This metric is mature and easy to calculate. However, it dosen't show the ability of method to determine whether an area has changed directly. So we use accuracy as an auxiliary metric to present that ability.

For the new task, another important aspect is dataset. The necessary data is semantic information label. Most existing datasets are for traditional binary change detection, which are not suitable for SCPA task. The most related one is the dataset proposed in \cite{daudt2019multitask}. Yet the images and labels of the dataset are acquired in different time from distinct sources, so the image and the label do not match well and there exist obvious errors in the label. That would severely influence model performance for the pixel-level task.

Therefore, we constructed a dataset via manually labeling pixel by pixel for proposed SCPA task. The dataset is based on a pair of large co-registered aerial images of Wuhan City, China, which is named as SCPA-Wuhan City(SCPA-WC) dataset. SCPA-WC is the first well-annotated aerial image dataset containing semantic information label for this task. We believe this publicly available dataset would facilitate the development of the newly proposed task.

On the whole, our contributions are summarized below:
\begin{itemize}
\item We propose a new task, semantic change pattern analysis, mainly for aerial images, and give the metric for the task, which is clean and interpretable. This is a higher-level task than traditional binary change detection or semantic change detection. Given a pair of images, it can not only decide where changes happen, but also determine the types of change. It would eliminate the ambiguity of change in traditional change detection and provide more useful and richer information for following automatic image analysis. 

\item We construct a dataset, SCPA-Wuhan City dataset, for the task. It contains a pair of large co-registered aerial image of Wuhan City, China, which is labeled manually pixel by pixel with semantic information. This is the first well-annotated aerial image dataset containing semantic information label for this task. We believe the publicly available dataset would make it convenient for people to try their new ideas in this field. 

\item We have conducted extensive baseline experiments on the dataset for the task. We have demonstrated the ability of current methods for the proposed task. These results would become the basis for future work, and encourage people to design more targeted methods for the task.  We hope it could facilitate the development of the task, and draw more attention from the community.
\end{itemize}

\section{Related Work}
\subsection{Binary Change Detection}
Binary change detection in this paper refers to the traditional change detection task. The task aims to identify changed areas for a pair of co-registered images. This is a fundamental topic for automatic image analysis, especially for aerial images, since it could provide information of land surface that has experienced changes. Many people have conducted research in this field.

Methods for binary change detection usually consist of two processes. The first one is to calculate a difference map between corresponding pixels, and the second one is to separate these pixels into ``change" or ``no change" based on a threshold \cite{daudt2019multitask}. These works either focus on the image differencing method \cite{bovolo2005wavelet,el2016convolutional,el2017zoom}, or put effort on decision function \cite{bruzzone2000automatic,celik2009unsupervised}. More recent works \cite{zhan2017change,chen2018mfcnet} take use of convolutional neural networks to perform binary change detection. This task is only able to determine whether a region has changes, without telling the type of change.
\subsection{Semantic Change Detection}
The concept of semantic change detection was first proposed in \cite{kataoka2016semantic}. \cite{alcantarilla2018street} proposed a network called CDnet to find structural changes in street view video. \cite{varghese2018changenet} presented a network named ChangeNet based on parallel deep convolutional neural network architecture for the task. This task is based on traditional change detection and mainly targets street view situation. The process contains two steps. The first step is to decide whether an area has changed, which is the same with binary change detection. The additional step is to give newly added objects a class label. 

To be specific, if a car appears in the destination image, the car area would be labeled as changed, and another class label of car would be assigned to the area. In other words, it only focuses on what the newly added object is, without considering the change type, i.e., the area changed from \emph{what} to a car. Although a later work \cite{daudt2019multitask} involves three kinds of change, city expansion, soil change, and water change, it focuses on objects and doesn't conform to the change type definition here. Because the semantic labels of the changed areas in source image and destination image are not specified at all. Take soil change for instance, it can not answer what an area used to be before it changed to soil, or what the soil changed to in the destination image.

\subsection{Change Detection Datasets}
Various datasets have been constructed in the field of change detection. For binary change detection, \cite{benedek2009change} constructed a binary change detection dataset named air change dataset with several aerial image pairs. \cite{daudt2018urban} presented a dataset called ONERA satellite change detection dataset, which is composed of some multispectral satellite image pairs. \cite{bourdis2011constrained} showed a dataset named aerial imagery change detection dataset. The dataset consists of synthetic aerial iamges with artificial changes generated by the rendering engine.

For semantic change detection, \cite{kataoka2016semantic} transformed the TSUNAMI dataset \cite{sakurada2015change} for the task via adding semantic label to the destination image. \cite{alcantarilla2018street} proposed the VL-CMU-CD dataset, which uses simultaneous localization and mapping technique to get nearly registered images. \cite{daudt2019multitask} built a dataset named HRSCD, whose images and labels come from different sources. The only guarantee is that the images and labels are acquired in the same year, which means the time difference between the images and labels could be as large as one full year. That would severely influence model's performance on the pixel-level task and make the result unreliable if we adopt the dataset. Therefore, a well-annotated dataset containing registered image pairs and corresponding semantic label is needed for SCPA task.

\section{Semantic Change Pattern Analysis Task}
\subsection{Task Definition}
\textbf{Task format.} The format for semantic change pattern analysis task is intuitive. If there are $L$ land classes encoded by $\mathcal{L}:=\{0,\dots,L-1\}$ in source image and destination image, there would be $N=L^2-L+1$ change types at most, where the last $1$ denotes the type of no change. The $N$ change types are encoded by $\mathcal{N}:=\{0,\dots,N-1\}$. For a pair of pixels $(x_i,y_i)$ belonging to the source image $X$ and registered destination image $Y$ respectively, the \emph{semantic change pattern analysis} task requires to map the pair of pixels to a change type $n_i\in\mathcal{N}$. The change type $n_i$ is jointly defined by $x_i$ and $y_i$, so difference in $x_i$ or $y_i$ would both result in a different change type $n_i$. We present a brief proof below. 
\newtheorem{suan}{\textbf Theorem}
\begin{suan}
	Given $L$ land classes in source image and destination image, the maximal number of corresponding change type is $N=L^2-L+1$.
\end{suan}
\begin{proof}
	For $L$ land classes in source image and destination image, we first randomly take one pixel $x_i$ from source image $X$, and the corresponding one $y_i$ from destination image $Y$. The possible combination situation would be $L\times L$. Then we define all kinds of no change situation, $\{l_1\to l_1, l_2\to l_2,\dots,l_L\to l_L\}, l_i\in \mathcal{L}$ as a whole one, which is reasonable for our task, so we need to remove $L-1$ from $L\times L$, and get the result $N=L^2-L+1$. 
\end{proof}

\noindent\textbf{Relation to change detection.} Semantic change pattern analysis is a higher-level task than traditional binary change detection and semantic change detection task. Compared to them, SCPA not only requires the information of whether an area has changed, but also the information of what kind of change type the area has experienced. In other words, the solution to SCPA problem would contain the solutions to binary change detection problem and semantic change detection problem. 

\noindent\textbf{Relation to semantic segmentation.} Semantic segmentation problem performs a pixel-level classification task for a single image input. Obviously, SCPA could also be regarded as a multi-class pixel-level classification task. The difference is the pixel label of SCPA is the change type, other than the land category, and it requires a pair of images as input, not single one. Yet the connection is still close. In fact, an intuitive approach to perform SCPA task is based on semantic segmentation method, which we would explain later in \ref{pipe}.

\noindent\textbf{Relation to status transition process.} Status transition process is a kind of random process where system status changes. SCPA task has close relation with the status transition problem, except that SCPA requires both source image and destination image as input, while status transition problem only needs source image as input and outputs the transition probability matrix for each land class.

\subsection{Task Metric}
\label{metric}
\noindent\textbf{Mean IoU.} Since SCPA is a multi-class pixel-level classification problem actually, where each class is a type of change, we adopt the mean Intersection over Union(mIoU) over all classes as the main metric. It makes the metric \emph{insensitive} to class imbalance. This metric is mature and easy to calculate. It's able to assess model's performance on determining each change type comprehensively. Formally, 
\begin{equation}
	{\rm mIoU}=\frac{1}{M}\sum_{i=0}^{M-1}\frac{C_{ii}}{\sum _{j=0}^{M-1}C_{ij}+\sum _{j=0}^{M-1}C_{ji}-C_{ii}}
\end{equation}
where $C_{ij}$ is the number of pixels of change type $i$ predicted to belong to change type $j$, and $M$ is the number of total change types that actually appears.

\noindent\textbf{Binary Accuracy.} Although mean IoU is able to evaluate model's performance on determining each change type, it can not directly present method's ability to decide whether an area has changed, which is also an important aspect for the task. To meet the requirement, we ignore the difference between all kinds of change types, and only care whether the pixel has changed or not. In other words, here SCPA is degraded to a binary change detection problem. Then we take binary accuracy(BAcc) as the auxiliary metric, the widely used one in binary change detection task. Specifically,
\begin{equation}
	{\rm BAcc=\frac{TP+TN}{TP+FP+TN+FN}}
\end{equation}
where TP denotes the number of pixels that method predicted as changed, and the corresponding ground truth is also labeled as changed, no matter what the specific change type is. TN represents the number of pixels that both the method and the ground truth denote as no change.

\begin{figure}[t]
	\centering{\includegraphics[width=0.95\columnwidth]{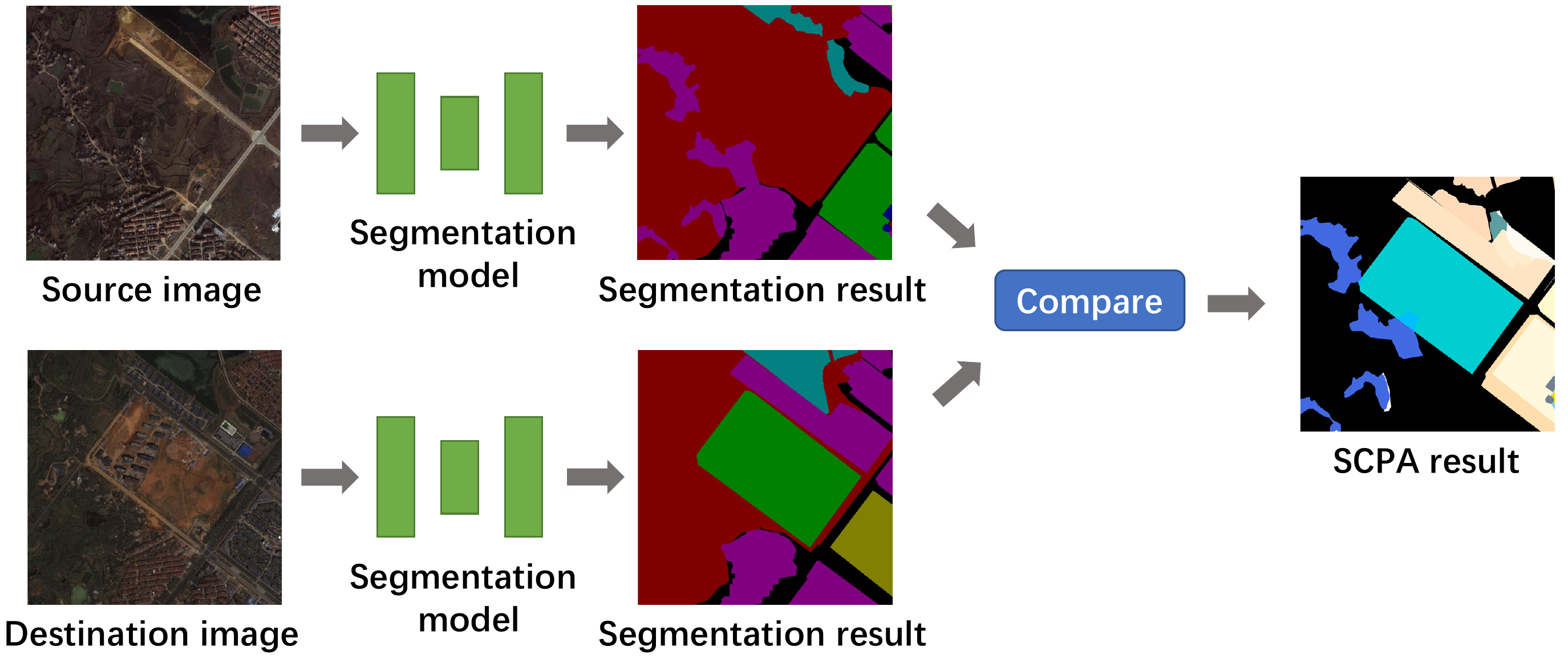}}
	\caption{The two-step method pipeline based on semantic segmentation models. The first step is to get the land class segmentation result for both source image and destination image through semantic segmentation method, the second step is to compare predicted source image label with destination image label, and then we obtain the final SCPA result.}
	\label{fig:pipeline}
\end{figure}

\subsection{Task Pipeline}
\label{pipe}
In this part, we would discuss two pipelines for SCPA task. One is a two-step method, which is based on semantic segmentation method. It enables us to fully take use of existing methods to deal with the SCPA task. The other one is a unified one-step method, which is supposed to target for the SCPA task.

\noindent\textbf{Two-step method.} Since SCPA aims to determine the change type for a pair of images, a natural thought would be a two-step approach. The first step is to get the land class label for both source image and destination image through semantic segmentation method, the second step is to compare predicted source image label with destination image label, and then we obtain the final result. The pipeline is shown in Fig. \ref{fig:pipeline}. It enables us to take use of existing methods to tackle the problem, which establishes a meaningful foundation for future work. We take this approach as a baseline in following experiments.

\begin{figure}[t]
	\centering{\includegraphics[width=0.75\columnwidth]{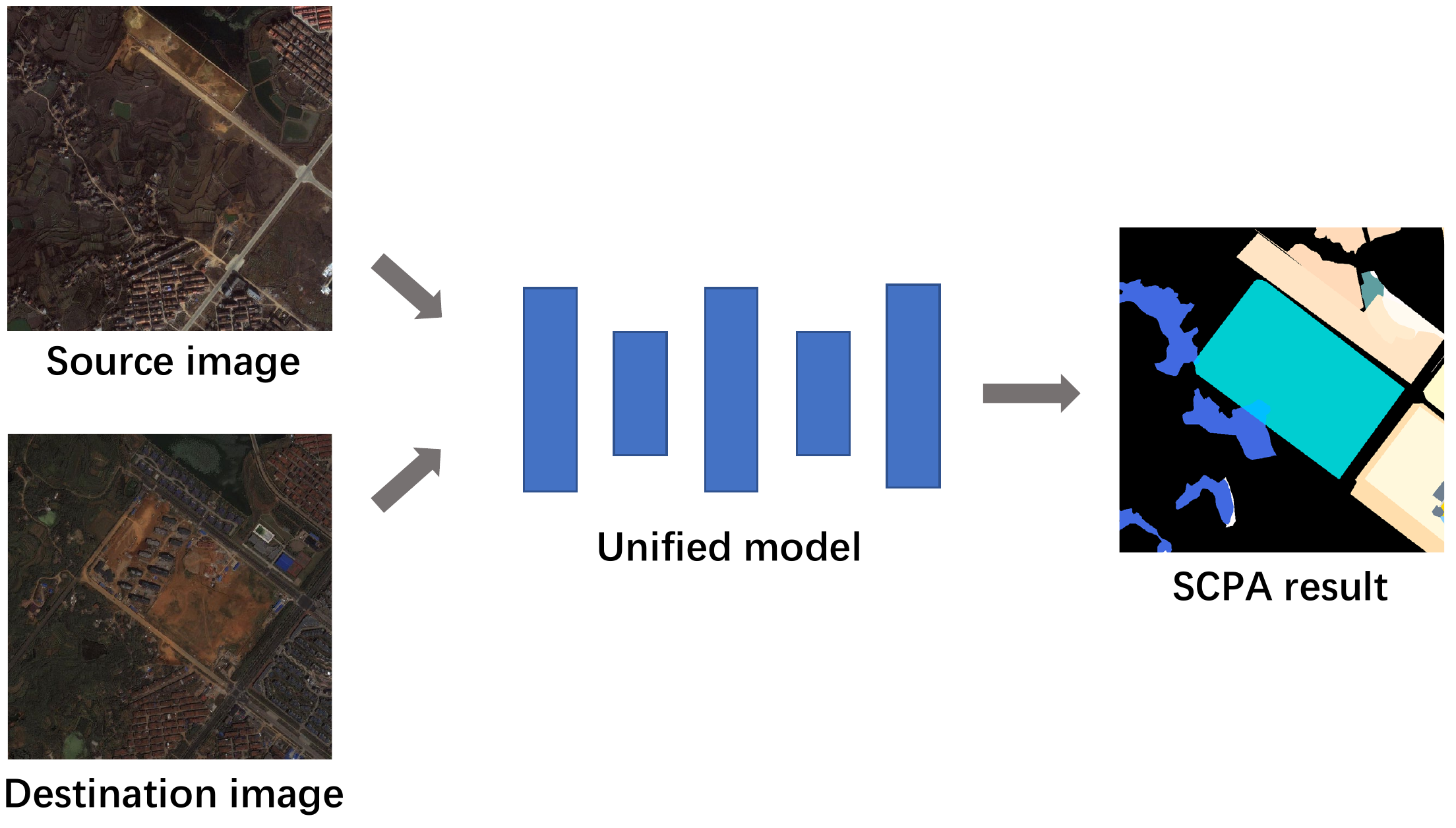}}
	\caption{The ideal one-step method pipeline. For a pair of images, the one-step method could directly give the SCPA result, without taking extra steps.}
	\label{fig:one_step}
\end{figure}

\noindent\textbf{One-step method.} Although two-step method mentioned above could be used to deal with SCPA task, it is not specialized designed for the task. For unchanged pixels, the land class information outputted by two-step method is not necessary for the task, since we only need to know the change type for actually changed pixels. Hence we expect a more unified, one-step method targeting SCPA task in the future. Ideally, given a pair of images, the one-step method could directly give the output of the task, without taking extra steps. Fig. \ref{fig:one_step} presents the pipeline. Some related works \cite{daudt2018fully,chen2018mfcnet} have been proposed in binary and semantic change detection field, yet it still requires more effort to construct a unified, one-step method for the SCPA task.


\section{SCPA-Wuhan City Dataset}
\subsection{Dataset Overview}
To facilitate the development of the proposed task, we construct the first well-annotated dataset containing both semantic label and change situation. The dataset is based on a large pair of aerial images of Wuhan City, China, which is named as SCPA-Wuhan City(SCPA-WC) dataset. The large pair of aerial images is first registered in pixel level to meet the requirement of the task. They were acquired by the IKONOS sensor, and have a spatial resolution of 1 m. The aerial images come from \cite{wu2017kernel}.

\begin{figure}[!t]
	\centering{\includegraphics[width=0.95\columnwidth]{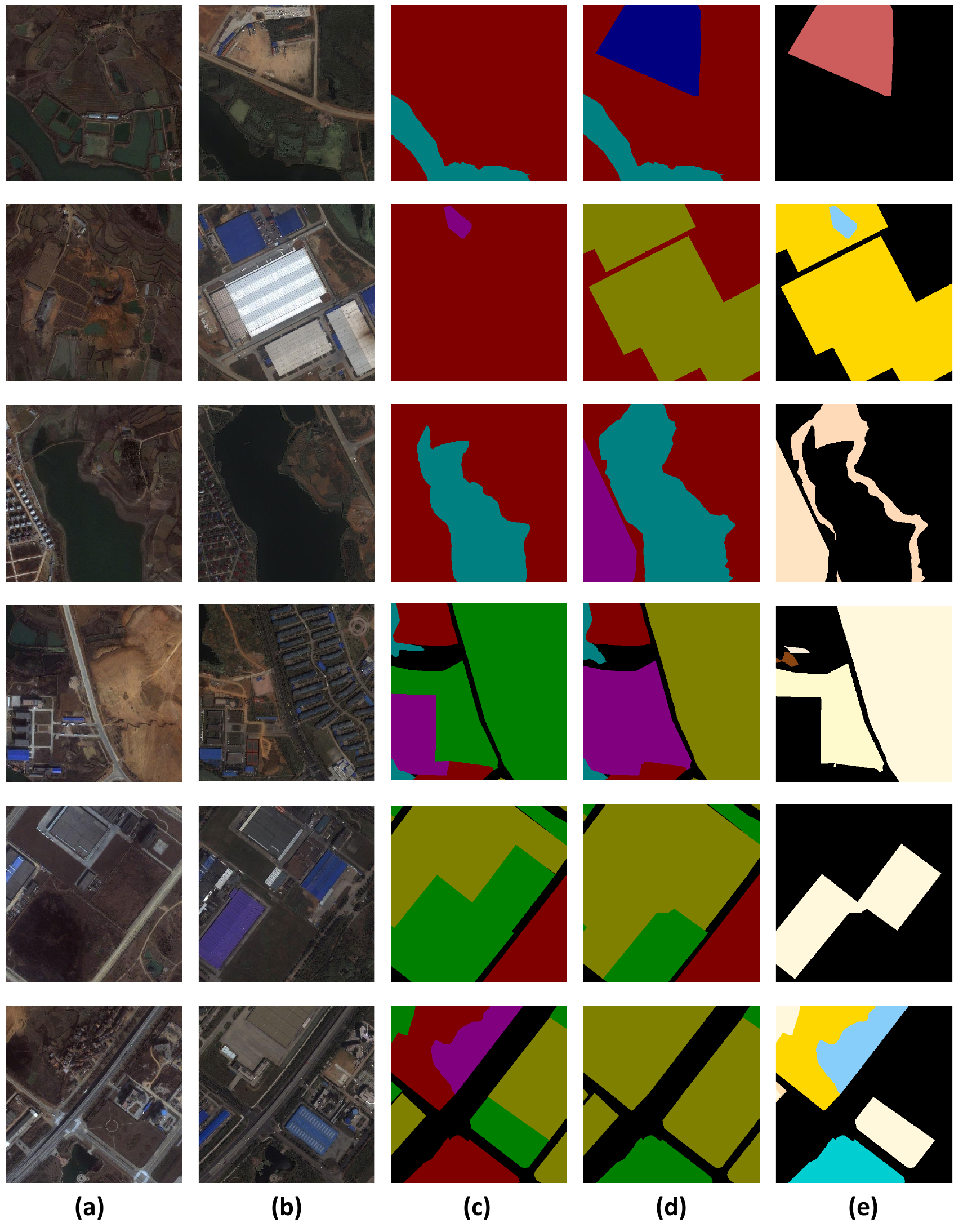}}
	\caption{SCPA-Wuhan City dataset samples. (a) and (b) are source image and destination image respectively. (c) and (d) are their pixel-wise land class labels. Specifically, black: background, red: farmland, green: bare land, yellow: industrial area, blue: parking area, purple: residential area, cyan: water body. (e) is the corresponding change type label. }
	\label{fig:data}
\end{figure}

In practice, it's difficult to directly label the change situation for each pixel, especially including each change type for the task. Therefore, we first label each pixel of source image and destination image respectively with its land class, i.e., give the label of semantic segmentation task, compare the pixel label in the two images, and then get the final label of change type. This way is efficient and accurate. To make sure the labels of unchanged areas keep same, we map the label of source image to destination image and adjust the label for changed areas after labeling the source image first. Besides, since this labeling approach provides extra land class label for unchanged areas, it would benefit the training process of two-step method mentioned before for the task with more training samples.

About data split, we obtain 1,706 image patches(853 pairs of images) with the size of 512$\times$512 by cropping the large pair of images. We then randomly extract 1/2 from these images as training set, 1/6 as validation set and 1/3 as test set.

\begin{figure}[!t]
	\centering{\includegraphics[width=0.95\columnwidth]{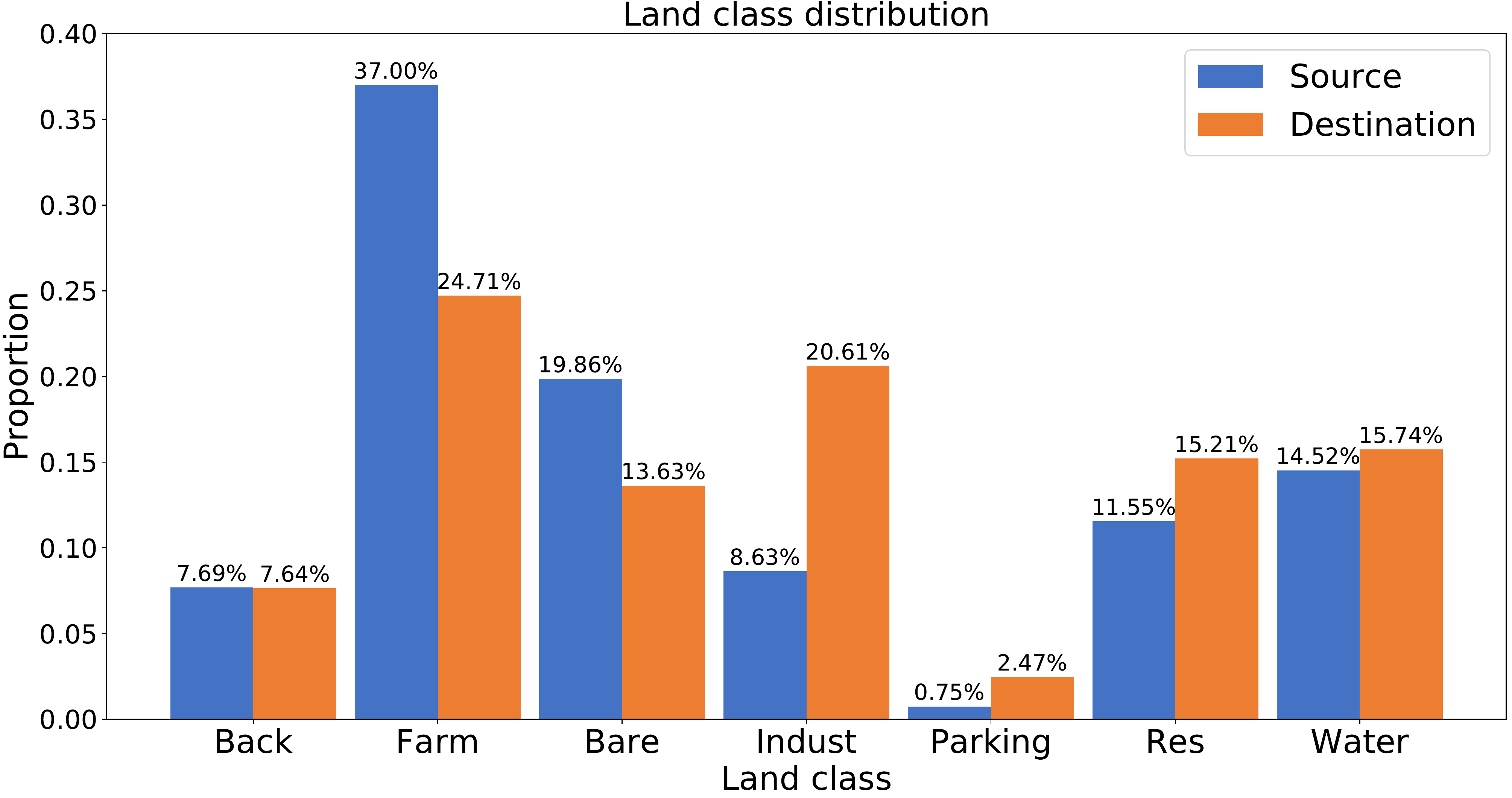}}
	\caption{The land class distribution of SCPA-WC dataset. The blue bar represents source image and the yellow bar denotes destination image. It presents the proportion of various land classes in corresponding image. Back: background, Farm: farmland, Bare: bare land, Indust: industrial area, Parking: parking area, Res: residential area, Water: water body.}
	\label{fig:count}
\end{figure}

\subsection{Dataset Property}
There are 7 land classes in total in this dataset. Specifically, they are background, farmland, bare land, industrial area, parking area, residential area, and water body, which are common categories in urban area. The change that happens among these categories also gets much attention in practice. Fig. \ref{fig:count} presents the detailed distribution of these categories. As we can see, farmland takes up the largest proportion in both source image and destination image. After several years, a large proportion of farmland and bare land has disappeared. On the contrary, industrial area, parking area and residential area have increased a lot.

We further calculate the distribution of all change types, as illustrated in Fig. \ref{fig:matrix}. Obviously, no change is the most common situation for the dataset. This is reasonable for a very large area. Also, it's interesting to look up the symmetric elements along diagonal. For instance, 1,498,180 pixels of farmland in the source image have changed to industrial area in destination image, and 321,957 pixels of bare land have changed to parking area. But the opposite situation didn't happen. This characterizes the development situation of the city, i.e., the city is at a fast development period, not in decaying period. However, current change detection task won't present this information. It effectively demonstrates the significance of proposed SCPA task in practical use.
\begin{figure}[!t]
	\centering{\includegraphics[width=0.95\columnwidth]{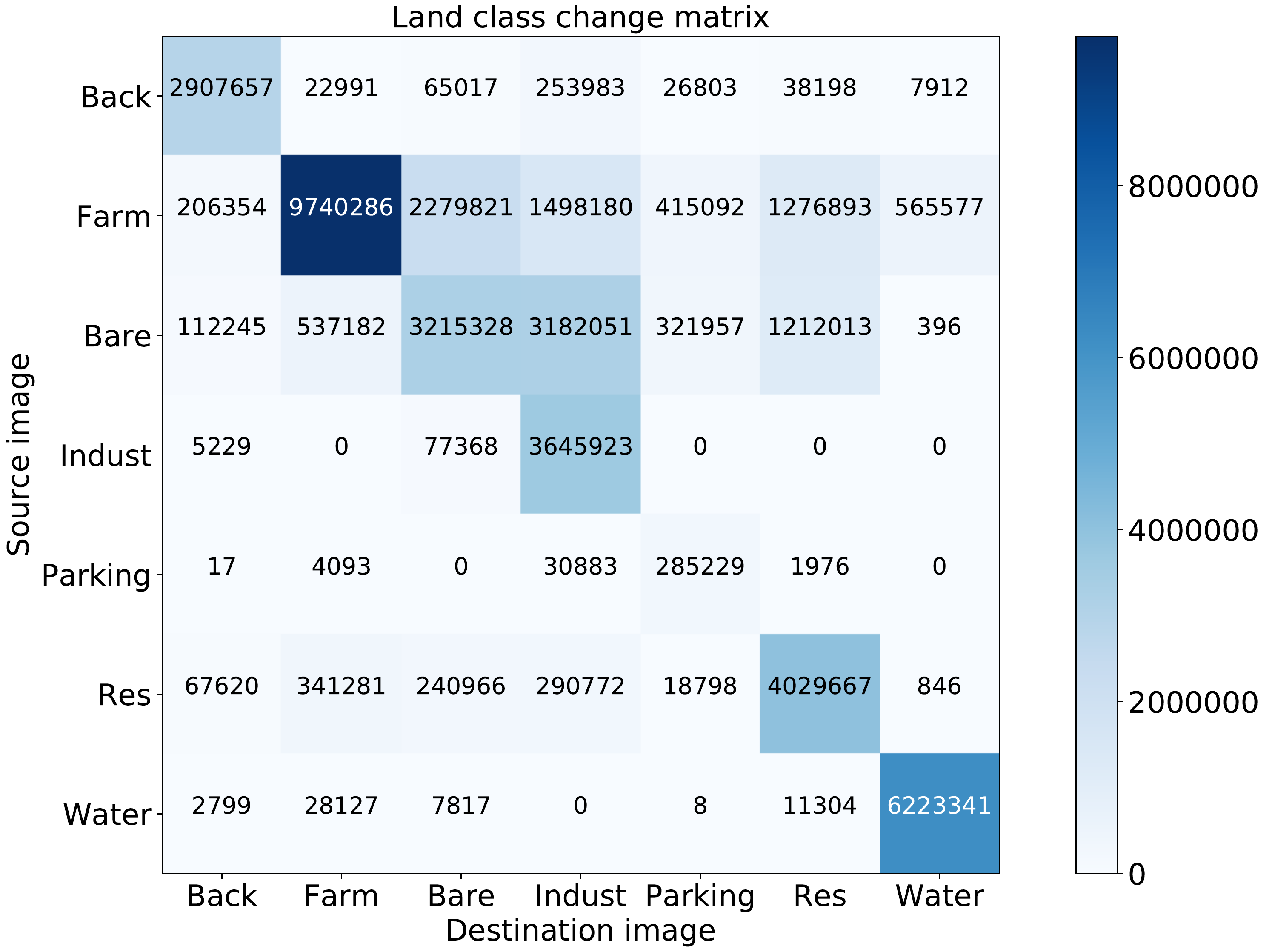}}
	\caption{The land class change situation matrix of SCPA-WC dataset. It presents the number of pixel pairs belonging to different change types. Note that pixel pairs on the diagonal has the same change type, i.e., no change. Back: background, Farm: farmland, Bare: bare land, Indust: industrial area, Parking: parking area, Res: residential area, Water: water body.}
	\label{fig:matrix}
\end{figure}

\section{Experiments}
\subsection{Experimental Settings}
We have conducted extensive baseline experiments on SCPA-WC dataset for the task. To fully take use of labeled data, we use training set and validation set together to train these models, and test set to evaluate their performance.

About training details, all models are trained on 4 Titan XP GPUs with 4$\times$12GB memory totally. Stochastic gradient descent (SGD) with momentum is adopted to train this network. Momentum value is set as 0.9. We apply poly learning rate policy to adjust learning rate, which reduces the learning rate per iteration. This could be expressed as:
\begin{equation}
	LR=initial\_lr\times (1-\frac{iter}{max\_iter})^{power}
\end{equation}
where $LR$ is the current learning rate, $initial\_lr$ is initial learning rate, $iter$ is the current iteration step, and $max\_iter$ is the maximal iteration step. The $initial\_lr$ is set as 0.01 and $power$ is set as 0.9. The $max\_iter$ depends on batch size and the number of epoch. For all models, batch size is 12 and the number of epoch is set to 100 to make sure these networks converge.

\begin{table}[t]
	\caption{Baseline results on SCPA-WC dataset. The middle part shows the class IoU for partial change types. BAcc is the accuracy regarding SCPA as a traditional binary change detection task, and mIoU is the mean IoU over all change types. NoC: no change, (0,1): the pixel changed from land class 0 to land class 1. \{0,1,2,3,4,5,6\} represents background, farmland, bare land, industrial area, parking area, residential area, water body respectively.}
	\begin{center}
		\begin{tabular}{l|cccccccccc|cc}
			\hline
			Methods                    &  NoC  & (0,1) & (0,2) & (0,3) & (0,4) & (0,5) & (0,6) & (1,0) & (1,2) & (1,3) & BAcc  & mIoU \\ \hline\hline
			FCN\cite{fcn}          & 64.74 & 0.00  & 0.81  & 2.51  & 4.47 & 0.00  & 0.00  & 8.14  & 28.17 & 42.49 & 73.73 & 9.59 \\
			FRRN-B\cite{frrnb}        & 61.26 & 0.00  & 1.43  & 2.69  & 2.29  & 0.00  & 0.00  & 7.27  & 25.15 & 32.16 & 70.81 & 8.31 \\
			GCN\cite{gcn}             & 63.81 & 0.18  & 1.14  & 2.35  & 1.95  & 0.00  & 0.00  & 9.20 & 30.62 & 41.01 & 72.84 & 9.41 \\
			RefineNet\cite{refinenet} & 64.11 & 0.00  & 0.10  & 2.38  & 3.20  & 0.00  & 0.00  & 11.74 & 27.12 & 41.97 & 73.25 & 9.85 \\
			Deeplabv3\cite{deeplabv3} & 64.23 & 0.01  & 1.09  & 2.60  & 3.96 & 0.00  & 0.00  & 8.14  & 27.32 & 45.49 & 72.87 & 9.92 \\
			\hline
		\end{tabular}
		\label{tab:rlt}
	\end{center}
\end{table}

\subsection{Baseline Methods}
We adopt two-step method mentioned above as the baseline in our experiments. As we explained before, it consists of two stages. The first stage is to get the land class label for both source image and destination image, which actually equals to semantic segmentation task. The other stage is to compare the land class label and get the corresponding change type. To make a comprehensive experiment, we have benchmarked many typical semantic segmentation methods on the SCPA-WC dataset.

\begin{figure}[!t]
	\centering{\includegraphics[width=0.95\columnwidth]{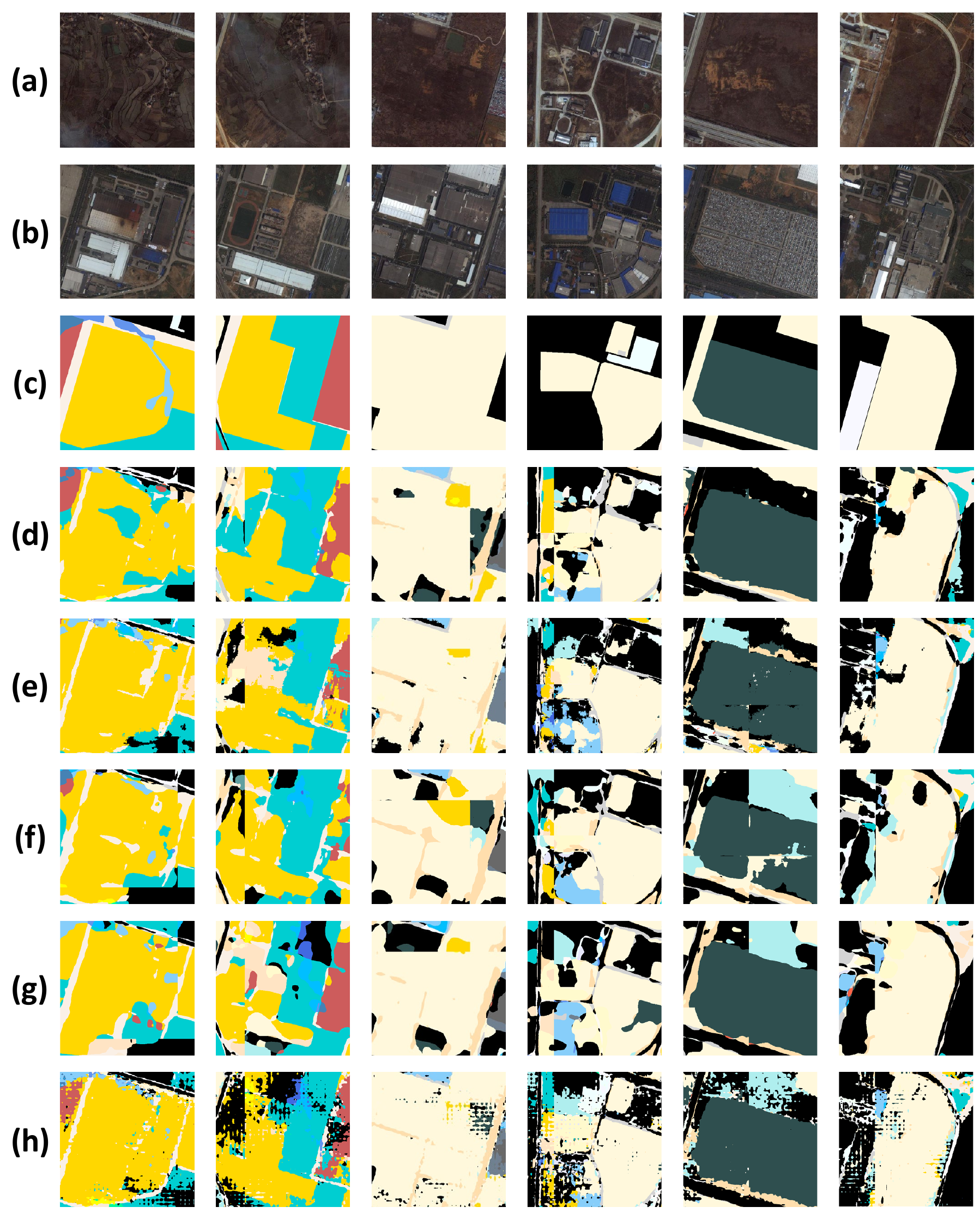}}
	\caption{Visualization results on SCPA-WC dataset. Each color represents a kind of change type. (a) is source image, (b) is destination image and (c) is ground truth. (d)-(h) are results of FCN, FRRN-B, GCN, RefineNet, Deeplabv3 respectively.}
	\label{fig:out}
\end{figure}

To be specific, these methods include FCN \cite{fcn}, FRRN-B \cite{frrnb}, GCN \cite{gcn},  DeepLabv3 \cite{deeplabv3} and RefineNet \cite{refinenet}. FCN is the first one dealing with semantic segmentation task with proposed fully convolutional network. FRRN designs two information streams to combine high resolution feature maps with low resolution ones. GCN demonstrats the importance of large convolutional kernel and utilizes it to improve segmentation performance.  DeepLabv3 adopts astrous spatial pyramid pooling module with global pooling operation to extract high resolution feature maps without adding extra parameter. RefineNet is a generic multi-path refinement network that explicitly exploits all the information along the down-sampling process. It enables high-resolution prediction with long-range residual connections.

\subsection{Experimental Analysis}
Tab. \ref{tab:rlt} presents quantitative results. Fig. \ref{fig:out} demonstrates visualization results of these methods. As Tab. \ref{tab:rlt} shows, these methods can not deal with the SCPA task well. The mean IoU is as low as 8.31\%. For change types (0,5) and (0,6), all methods fail totally. This indicates the simple two-step pipeline based on semantic segmentation only is far from satisfactory. Hence there is much room to create new methods for the challenging SCPA task.

Besides, as we can see in Tab. \ref{tab:rlt}, the binary accuracy metric is high. This means these methods can handle traditional binary change detection task well. On the contrary, SCPA task is really challenging for them. That proves SCPA is a higher-level and more complex task than traditional change detection, which is worth exploration and study.


\section{Future of Semantic Change Pattern Analysis}
Semantic change pattern analysis is a meaningful and challenging task. Current methods are not able to handle the task well. To facilitate the development of the filed, we would like to conclude by discussing some possibilities.

\textbf{Dataset.} Although this work provides the first well-annotated dataset for SCPA task, it only focuses on aerial images. Datasets in other fields, like street view are also expected. These datasets would broaden SCPA task's usage and benefit the evaluation of SCPA models in those fields.

\textbf{Method.} As the experiment shows, two-step method based on semantic segmentation only performs poor on SCPA task. This encourages us to explore other ways. On one hand, for two-step method, we think the main reason why these naive ways fail is they don't take the relation between the source image and destination image into consideration. Hence we could try to introduce spatial-temporal relation to the problem. For instance, status transition probability could be used as a auxiliary information to decide the change type. Similar thought has been studied for binary change detection task in \cite{liu2008using}. On the other hand, we also expect a more unified method, which can handle both source image and destination image simultaneously and give the final SCPA result directly and neatly. We think this is really interesting though challenging. Siamese network might be useful.

\textbf{Application.} Though we mainly focus on aerial images in this work, the application of SCPA should not be limited to this field. It could be applied to general scenes including street view, just like existing usages of change detection in general vision field \cite{eden2008using,pichaikuppan2014change,ulusoy2014image,sakurada2015change,alcantarilla2018street}. Besides, we think visual tracking task and SCPA might benefit each other, because spatial-temporal relation is an important part for both tasks. We hope semantic change pattern analysis task could draw more attention from the community, and invigorate the change detection and related fields.


%
%
\bibliographystyle{splncs04}
\bibliography{segmentation,change_detection}

\begin{thebibliography}{10}
\providecommand{\url}[1]{\texttt{#1}}
\providecommand{\urlprefix}{URL }
\providecommand{\doi}[1]{https://doi.org/#1}

\bibitem{alcantarilla2018street}
Alcantarilla, P.F., Stent, S., Ros, G., Arroyo, R., Gherardi, R.: Street-view
  change detection with deconvolutional networks. Autonomous Robots
  \textbf{42}(7),  1301--1322 (2018)

\bibitem{benedek2009change}
Benedek, C., Szir{\'a}nyi, T.: Change detection in optical aerial images by a
  multilayer conditional mixed markov model. IEEE Transactions on Geoscience
  and Remote Sensing  \textbf{47}(10),  3416--3430 (2009)

\bibitem{bilodeau2013change}
Bilodeau, G.A., Jodoin, J.P., Saunier, N.: Change detection in feature space
  using local binary similarity patterns. In: 2013 International Conference on
  Computer and Robot Vision. pp. 106--112. IEEE (2013)

\bibitem{bourdis2011constrained}
Bourdis, N., Marraud, D., Sahbi, H.: Constrained optical flow for aerial image
  change detection. In: 2011 IEEE International Geoscience and Remote Sensing
  Symposium. pp. 4176--4179. IEEE (2011)

\bibitem{bovolo2005wavelet}
Bovolo, F., Bruzzone, L.: A wavelet-based change-detection technique for
  multitemporal sar images. In: International Workshop on the Analysis of
  Multi-Temporal Remote Sensing Images, 2005. pp. 85--89. IEEE (2005)

\bibitem{bruzzone2012novel}
Bruzzone, L., Bovolo, F.: A novel framework for the design of change-detection
  systems for very-high-resolution remote sensing images. Proceedings of the
  IEEE  \textbf{101}(3),  609--630 (2012)

\bibitem{bruzzone2000automatic}
Bruzzone, L., Prieto, D.F.: Automatic analysis of the difference image for
  unsupervised change detection. IEEE Transactions on Geoscience and Remote
  sensing  \textbf{38}(3),  1171--1182 (2000)

\bibitem{celik2009unsupervised}
Celik, T.: Unsupervised change detection in satellite images using principal
  component analysis and $ k $-means clustering. IEEE Geoscience and Remote
  Sensing Letters  \textbf{6}(4),  772--776 (2009)

\bibitem{deeplabv3}
Chen, L.C., Papandreou, G., Schroff, F., Adam, H.: Rethinking atrous
  convolution for semantic image segmentation. arXiv preprint arXiv:1706.05587
  (2017)

\bibitem{chen2018mfcnet}
Chen, Y., Ouyang, X., Agam, G.: Mfcnet: End-to-end approach for change
  detection in images. In: 2018 25th IEEE International Conference on Image
  Processing (ICIP). pp. 4008--4012. IEEE (2018)

\bibitem{daudt2018fully}
Daudt, R.C., Le~Saux, B., Boulch, A.: Fully convolutional siamese networks for
  change detection. In: 2018 25th IEEE International Conference on Image
  Processing (ICIP). pp. 4063--4067. IEEE (2018)

\bibitem{daudt2018urban}
Daudt, R.C., Le~Saux, B., Boulch, A., Gousseau, Y.: Urban change detection for
  multispectral earth observation using convolutional neural networks. In: 2018
  IEEE International Geoscience and Remote Sensing Symposium. pp. 2115--2118.
  IEEE (2018)

\bibitem{daudt2019multitask}
Daudt, R.C., Le~Saux, B., Boulch, A., Gousseau, Y.: Multitask learning for
  large-scale semantic change detection. Computer Vision and Image
  Understanding  \textbf{187},  102783 (2019)

\bibitem{eden2008using}
Eden, I., Cooper, D.B.: Using 3d line segments for robust and efficient change
  detection from multiple noisy images. In: European Conference on Computer
  Vision (ECCV). pp. 172--185. Springer (2008)

\bibitem{el2016convolutional}
El~Amin, A.M., Liu, Q., Wang, Y.: Convolutional neural network features based
  change detection in satellite images. In: First International Workshop on
  Pattern Recognition. vol. 10011, p. 100110W. International Society for Optics
  and Photonics (2016)

\bibitem{el2017zoom}
El~Amin, A.M., Liu, Q., Wang, Y.: Zoom out cnns features for optical remote
  sensing change detection. In: 2017 2nd International Conference on Image,
  Vision and Computing (ICIVC). pp. 812--817. IEEE (2017)

\bibitem{feng2015fine}
Feng, W., Tian, F.P., Zhang, Q., Zhang, N., Wan, L., Sun, J.: Fine-grained
  change detection of misaligned scenes with varied illuminations. In:
  Proceedings of the IEEE International Conference on Computer Vision (ICCV).
  pp. 1260--1268 (2015)

\bibitem{goyette2012changedetection}
Goyette, N., Jodoin, P.M., Porikli, F., Konrad, J., Ishwar, P.:
  Changedetection. net: A new change detection benchmark dataset. In:
  Proceedings of the IEEE Conference on Computer Vision and Pattern Recognition
  (CVPR) Workshops. pp.~1--8. IEEE (2012)

\bibitem{gressin2013semantic}
Gressin, A., Vincent, N., Mallet, C., Paparoditis, N.: Semantic approach in
  image change detection. In: International Conference on Advanced Concepts for
  Intelligent Vision Systems. pp. 450--459. Springer (2013)

\bibitem{kataoka2016semantic}
Kataoka, H., Shirakabe, S., Miyashita, Y., Nakamura, A., Iwata, K., Satoh, Y.:
  Semantic change detection with hypermaps. arXiv preprint arXiv:1604.07513
  \textbf{2}(4) (2016)

\bibitem{refinenet}
Lin, G., Milan, A., Shen, C., Reid, I.: Refinenet: Multi-path refinement
  networks for high-resolution semantic segmentation. In: Proceedings of the
  IEEE Conference on Computer Vision and Pattern Recognition (CVPR). pp.
  1925--1934 (2017)

\bibitem{liu2008using}
Liu, D., Song, K., Townshend, J.R., Gong, P.: Using local transition
  probability models in markov random fields for forest change detection.
  Remote Sensing of Environment  \textbf{112}(5),  2222--2231 (2008)

\bibitem{fcn}
Long, J., Shelhamer, E., Darrell, T.: Fully convolutional networks for semantic
  segmentation. In: Proceedings of the IEEE Conference on Computer Vision and
  Pattern Recognition (CVPR). pp. 3431--3440 (2015)

\bibitem{park2019robust}
Park, D.H., Darrell, T., Rohrbach, A.: Robust change captioning. In:
  Proceedings of the IEEE International Conference on Computer Vision (ICCV).
  pp. 4624--4633 (2019)

\bibitem{gcn}
Peng, C., Zhang, X., Yu, G., Luo, G., Sun, J.: Large kernel matters--improve
  semantic segmentation by global convolutional network. In: Proceedings of the
  IEEE Conference on Computer Vision and Pattern Recognition (CVPR). pp.
  4353--4361 (2017)

\bibitem{pichaikuppan2014change}
Pichaikuppan, V.R.A., Narayanan, R.A., Rangarajan, A.: Change detection in the
  presence of motion blur and rolling shutter effect. In: European Conference
  on Computer Vision (ECCV). pp. 123--137. Springer (2014)

\bibitem{frrnb}
Pohlen, T., Hermans, A., Mathias, M., Leibe, B.: Full-resolution residual
  networks for semantic segmentation in street scenes. In: Proceedings of the
  IEEE Conference on Computer Vision and Pattern Recognition (CVPR). pp.
  4151--4160 (2017)

\bibitem{revaud2019did}
Revaud, J., Heo, M., Rezende, R.S., You, C., Jeong, S.G.: Did it change?
  learning to detect point-of-interest changes for proactive map updates. In:
  Proceedings of the IEEE Conference on Computer Vision and Pattern Recognition
  (CVPR). pp. 4086--4095 (2019)

\bibitem{sakurada2015change}
Sakurada, K., Okatani, T.: Change detection from a street image pair using cnn
  features and superpixel segmentation. In: BMVC. pp. 61--1 (2015)

\bibitem{taneja2013city}
Taneja, A., Ballan, L., Pollefeys, M.: City-scale change detection in cadastral
  3d models using images. In: Proceedings of the IEEE Conference on computer
  Vision and Pattern Recognition (CVPR). pp. 113--120 (2013)

\bibitem{ulusoy2014image}
Ulusoy, A.O., Mundy, J.L.: Image-based 4-d reconstruction using 3-d change
  detection. In: European Conference on Computer Vision (ECCV). pp. 31--45.
  Springer (2014)

\bibitem{varghese2018changenet}
Varghese, A., Gubbi, J., Ramaswamy, A., Balamuralidhar, P.: Changenet: a deep
  learning architecture for visual change detection. In: Proceedings of the
  European Conference on Computer Vision (ECCV) workshop (2018)

\bibitem{wang2014cdnet}
Wang, Y., Jodoin, P.M., Porikli, F., Konrad, J., Benezeth, Y., Ishwar, P.:
  Cdnet 2014: An expanded change detection benchmark dataset. In: Proceedings
  of the IEEE Conference on Computer Vision and Pattern Recognition (CVPR)
  Workshops. pp. 387--394 (2014)

\bibitem{wu2017kernel}
Wu, C., Zhang, L., Du, B.: Kernel slow feature analysis for scene change
  detection. IEEE Transactions on Geoscience and Remote Sensing
  \textbf{55}(4),  2367--2384 (2017)

\bibitem{zhan2017change}
Zhan, Y., Fu, K., Yan, M., Sun, X., Wang, H., Qiu, X.: Change detection based
  on deep siamese convolutional network for optical aerial images. IEEE
  Geoscience and Remote Sensing Letters  \textbf{14}(10),  1845--1849 (2017)

\end{thebibliography}
\end{document}